%% file: jam-paper.tex
\documentclass[12pt]{article}
\input{preamble}

\title{\bf A Theoretical Analysis of Joint Manifolds}

\author{\em Mark A. Davenport, Chinmay Hegde, Marco F. Duarte, \\
\em and Richard G. Baraniuk\footnote{Department of Electrical and
Computer Engineering, Rice University, Houston, TX. Email: \{md,
chinmay, duarte, richb\}@rice.edu. \protect\\ \indent This work was
supported by grants NSF CCF-0431150 and CCF-0728867, DARPA/ONR
N66001-08-1-2065, ONR N00014-07-1-0936 and N00014-08-1-1112, AFOSR
FA9550-07-1-0301, ARO MURI W911NF-07-1-0185, and the TI Leadership
Program.
    \protect\\ \indent Web: dsp.rice.edu/cs} \protect\\\protect\\
    Rice University \protect\\ Department of Electrical and Computer Engineering\protect\\
    Technical Report TREE0901}

\begin{document}
\date{January 6, 2009}

\maketitle

\input{abstract}
\input{introduction}
\input{jam}
\input{processing}

\input{randproj}
\input{discussion}

\bibliographystyle{unsrt}

\end{document}

%% file: preamble.tex
\usepackage{graphicx}
\usepackage{color}
\usepackage{times}
\usepackage{epsfig}
\usepackage{graphicx}
\usepackage{amsmath}
\usepackage{comment}
\usepackage{amsthm}
\usepackage{amssymb}
\usepackage{amsfonts}
\usepackage{fullpage}
\usepackage{cite}
\newtheorem{thm}{Theorem}[section] %(If you want theorem numbered
\newtheorem{cor}{Corollary}[section]
\newtheorem{lemma}{Lemma}[section]
\newtheorem{prop}{Proposition}[section]
\newtheorem{definition}{Definition}[section]

\newcommand{\bitem}{\begin{itemize}}
\newcommand{\eitem}{\end{itemize}}

\newcommand{\beqn}{\begin{equation}}
\newcommand{\eeqn}{\end{equation}}
\newcommand{\balign}{\begin{align}}
\newcommand{\ealign}{\end{align}}

\def\reals    { \mathbb{R} }

\def \manifold { {\mathcal{M}} } % Manifold
\def \manifoldN { {\mathcal{N}} } % Manifold
\def \condition { {\tau} }
\def \jacobian { \mathcal{J} }

\def \dim {N}      % Dimension of signals
\newcommand{\iprod}[2]{\left\langle #1 , #2 \right\rangle}

%% file: abstract.tex
\begin{abstract}
The emergence of low-cost sensor architectures for diverse
modalities has made it possible to deploy sensor arrays that capture
a single event from a large number of vantage points and using
multiple modalities.  In many scenarios, these sensors acquire very
high-dimensional data such as audio signals, images, and video.  To
cope with such high-dimensional data, we typically rely on
low-dimensional models. Manifold models provide a particularly
powerful model that captures the structure of high-dimensional data
when it is governed by a low-dimensional set of parameters. However,
these models do not typically take into account dependencies among
multiple sensors.  We thus propose a new {\em joint manifold}
framework for data ensembles that exploits such dependencies. We
show that simple algorithms can exploit the joint manifold structure
to improve their performance on standard signal processing
applications. Additionally, recent results concerning dimensionality
reduction for manifolds enable us to formulate a network-scalable
data compression scheme that uses random projections of the sensed
data. This scheme efficiently fuses the data from all sensors
through the addition of such projections, regardless of the data
modalities and dimensions.
\end{abstract}

%% file: introduction.tex
\section{Introduction}
\label{sec:intro}

The geometric notion of a low-dimensional manifold is a common, yet
powerful, tool for modeling high-dimensional data.
Manifold models arise in cases where ({\em i}) a $K$-dimensional
parameter $\theta$ can be identified that carries the relevant
information about a signal and ({\em ii}) the signal $x_\theta \in
\reals^N$ changes as a continuous (typically nonlinear) function of
these parameters. Some typical examples include a one-dimensional
(1-D) signal shifted by an unknown time delay (parameterized by the
translation variable), a recording of a speech signal (parameterized by
the underlying phonemes spoken by the speaker), and an image of a
3-D object at an unknown location captured from an unknown viewing
angle (parameterized by the 3-D coordinates of the object and its roll,
pitch, and yaw).  In these and many other cases, the geometry of the
signal class forms a nonlinear $K$-dimensional manifold in
$\reals^N$,
\begin{equation} \label{eq:manifold_def}
\manifold = \{ f(\theta) : \theta \in \Theta \},
\end{equation}
where $\Theta$ is the $K$-dimensional parameter
space~\cite{DonohoGrimesISOMAP,GrimesThesis,WakinSPIEiam}.
Low-dimensional manifolds have also been proposed as approximate
models for nonparametric signal classes such as images of human
faces or handwritten digits~\cite{Eigenfaces,digits,broomhead01wh}.

In many scenarios, multiple observations of the same event may be
performed simultaneously, resulting in the acquisition of multiple
manifolds that share the same parameter space.  For example, sensor
networks --- such as camera networks or microphone arrays --- typically
observe a single event from a variety of vantage points, while the
underlying phenomenon can often be described by a set of common
global parameters (such as the location and orientation of the objects of
interest). Similarly, when sensing a single phenomenon using multiple
modalities, such as video and audio, the underlying phenomenon may
again be described by a single parameterization that spans all
modalities.  In such cases, we will show that it is
advantageous to model this joint structure contained in the ensemble
of manifolds as opposed to simply treating each manifold
independently. Thus we introduce the concept of the {\em joint
manifold}: a model for the concatenation of the data vectors observed
by the group of sensors. Joint manifolds enable the development of
improved manifold-based learning and estimation algorithms that
exploit this structure. Furthermore, they can be applied to data of
any modality and dimensionality.

In this work we conduct a careful examination of the theoretical
properties of joint manifolds. In particular, we compare joint
manifolds to their component manifolds to see how quantities like
geodesic distances, curvature, branch separation, and condition
number are affected.  We then observe that these properties lead to
improved performance and noise-tolerance for a variety of signal
processing algorithms when they exploit the joint manifold
structure, as opposed to processing data from each manifold
separately.  We also illustrate how this joint manifold structure
can be exploited through a simple and efficient data fusion
algorithm that uses random projections, which can also be applied to
multimodal data.

Related prior work has studied {\em manifold alignment}, where the
goal is to discover maps between several datasets that are governed
by the same underlying low-dimensional structure. Lafon et al.\ proposed an algorithm to obtain a one-to-one matching
between data points from several manifold-modeled
classes~\cite{lafon}. The algorithm first applies dimensionality
reduction using diffusion maps to obtain data representations that
encode the intrinsic geometry of the class. Then, an affine function
that matches a set of landmark points is computed and applied to the
remainder of the datasets. This concept was extended by Wang and
Mahadevan, who apply Procrustes analysis on the
dimensionality-reduced datasets to obtain an alignment function
between a pair of manifolds~\cite{procrustes}. Since an alignment
function is provided instead of a data point matching, the mapping
obtained is applicable for the entire manifold rather than for the
set of sampled points. In our setting, we assume that either ({\em i})
the manifold alignment is provided intrinsically via synchronization
between the different sensors or ({\em ii}) the manifolds have been
aligned using one of the approaches described above. Our main focus
is a theoretical analysis of the benefits provided by analyzing
the joint manifold versus solving our task of interest separately on
each of the manifolds observed by individual sensors.

This paper is organized as follows. Section~\ref{sec:jam} introduces
and establishes some basic properties of joint manifolds.
Section~\ref{sec:processing} considers the application of joint
manifolds to the tasks of classification and manifold learning.
Section~\ref{sec:defjam} then describes an efficient method for
processing and aggregating data when it lies on a joint manifold,
and Section~\ref{sec:conc} concludes with discussion.

%% file: jam.tex
\section{Joint manifolds}
\label{sec:jam}

In this section we develop a theoretical framework for ensembles of
manifolds which are {\em jointly} parameterized by a small number of
{\em common} degrees of freedom. Informally, we propose a data
structure for jointly modeling such ensembles; this is obtained by
concatenating points from different ensembles that are indexed by
the same articulation parameter to obtain a single point in a
higher-dimensional space. We begin by defining the joint manifold
for the general setting of arbitrary topological manifolds\footnote{A comprehensive introduction of topological manifolds can be found in Boothby~\cite{Boothby}.}.

\begin{definition}
Let $\manifold_1,\manifold_2,\ldots,\manifold_J$ be an ensemble of
$J$ topological manifolds of equal dimension $K$.  Suppose that the
manifolds are homeomorphic to each other, in which case there exists
a homeomorphism $\psi_{j}$ between $\manifold_1$ and $\manifold_j$
for each $j$.  For a particular set of mappings $\{ \psi_j
\}_{j=2}^J$, we define the {\bf joint manifold} as
$$
\manifold^* = \{(p_1,p_2,\ldots,p_J) \in \manifold_1 \times
\manifold_2 \times \cdots \times \manifold_J : p_j = \psi_j(p_1), 2
\le j \le J \}.
$$
Furthermore, we say that
$\manifold_1,\manifold_2,\ldots,\manifold_J$ are the corresponding
{\bf component manifolds}.
\end{definition}

Notice that $\manifold_1$ serves as a common {\em parameter space}
for all the component manifolds.  Since the component manifolds are
homeomorphic to each other, this choice is ultimately arbitrary.  In
practice it may be more natural to think of each component manifold
as being homeomorphic to some fixed $K-$dimensional parameter space $\Theta$.
However, in this case one could still define $\manifold^*$ as is
done above by defining $\psi_j$ as the composition of the
homeomorphic mappings from $\manifold_1$ to $\Theta$ and from
$\Theta$ to $\manifold_j$.

As an example, consider the one-dimensional manifolds in Figure
\ref{fig:helix}.  Figures \ref{fig:helix} (a) and (b) show two
isomorphic manifolds, where $\manifold_1 = (0, 2\pi)$ is an open
interval, and $\manifold_2 = \{ \psi_2(\theta) : \theta \in
\manifold_1 \}$ where $\psi_2(\theta) = (\cos(\theta),
\sin(\theta))$, i.e., $\manifold_2 = S^1 \backslash (1,0)$ is a
circle with one point removed (so that it remains isomorphic to a
line segment).  In this case the joint manifold $\manifold^* = \{
(\theta, \cos(\theta), \sin(\theta)) : \theta \in (0,2\pi) \}$,
illustrated in Figure \ref{fig:helix} (c), is a helix.  Notice that
there exist other possible homeomorphic mappings from $\manifold_1$
to $\manifold_2$, and that the precise structure of the joint
manifold as a submanifold of $\reals^3$ is heavily dependent on the
choice of this mapping.

\begin{figure}[!t]
\centering
\begin{tabular}{ccc}
{\includegraphics[width=0.3\hsize]{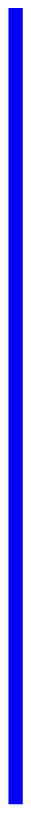}} &
{\includegraphics[width=0.3\hsize]{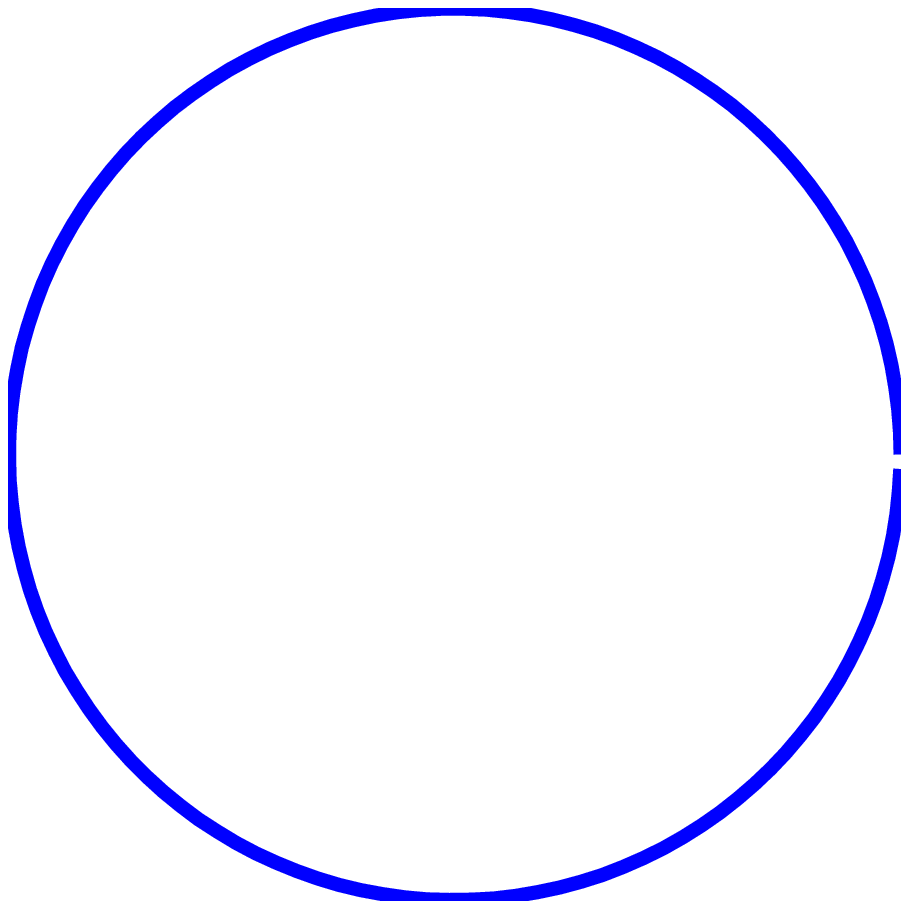}} &
{\includegraphics[width=0.4\hsize, viewport=50 60 400 300]{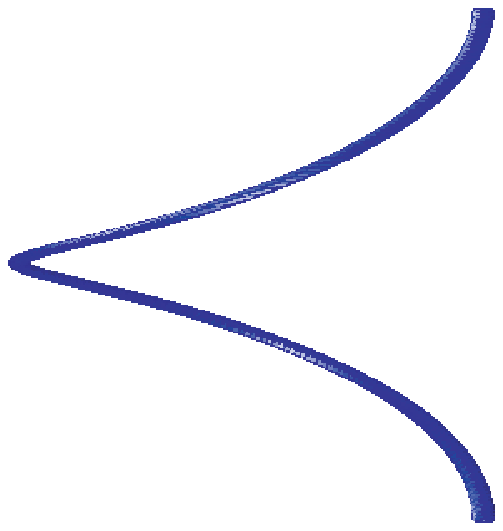}} \\
(a) $\manifold_1 \subseteq \reals$: line segment & (b) $\manifold_2 \subseteq \reals^2$:
circle segment & (c) $\manifold^* \subseteq \reals^3$: helix segment
\end{tabular}
\caption{\small\sl A pair of isomorphic manifolds $\manifold_1$ and
$\manifold_2$, and the resulting joint manifold $\manifold^*$.}
\label{fig:helix}
\end{figure}

Returning to the definition of $\manifold^*$, observe that although
we have called $\manifold^*$ the joint manifold, we have not shown
that it actually forms a topological manifold. To prove that
$\manifold^*$ is indeed a manifold, we will make use of the fact
that the joint manifold is a subset of the {\em product manifold}
$\manifold_1 \times \manifold_2 \times \cdots \times \manifold_J$.
One can show that the product manifold forms a $JK$-dimensional manifold
using the product topology~\cite{Boothby}. By comparison, we now
show that $\manifold^*$ has dimension only $K$.

\begin{prop} \label{prop:manifold}
$\manifold^*$ is a $K$-dimensional submanifold of $\manifold_1
\times \manifold_2 \times \cdots \times \manifold_J$.
\end{prop}
\begin{proof}
We first observe that since $\manifold^*$ is a subset of the product
manifold, we automatically have that $\manifold^*$ is a second
countable Hausdorff topological space.  Thus, all that remains is to
show that $\manifold^*$ is locally homeomorphic to $\reals^K$.  Let
$p = (p_1, p_2, \ldots, p_J)$ be an arbitrary point on
$\manifold^*$.  Since $p_1 \in \manifold_1$, we have a pair $(U_1,
\phi_1)$ such that $U_1 \subset \manifold_1$ is an open set
containing $p_1$ and $\phi_1 : U_1 \rightarrow V$ is a homeomorphism
where $V$ is an open set in $\reals^K$.  We now define for $2 \le j \le J$ $U_j =
\psi_j(U_1)$ and $\phi_j = \phi_1 \circ \psi_j^{-1} : U_j
\rightarrow V$.  Note that for each $j$, $U_j$ is an open set and
$\phi_j$ is a homeomorphism (since $\psi_j$ is a homeomorphism).

Now define $U^* = (U_1 \times U_2 \times \cdots \times U_J) \cap
\manifold^*$.  Observe that $U^*$ is an open set and that $p \in
U^*$.  Furthermore, let $q = (q_1,q_2, \ldots, q_J)$ be any element
of $U^*$. Then $\phi_j(q_j) = \phi_1 \circ \psi_j^{-1}(q_j) = \phi_1
(q_1)$ for each $2 \le j \le J$. Thus, since the image of each $q_j
\in U_j$ in $V$ under their corresponding $\phi_j$ is the same, we
can form a single homeomorphism $\phi^*: U^* \to V$ by assigning
$\phi^*(q) = \phi_1(q_1)$. This shows that $\manifold^*$ is locally
homeomorphic to $\reals^K$ as desired.
\end{proof}

Since $\manifold^*$ is a submanifold of $\manifold_1 \times
\manifold_2 \times \cdots \times \manifold_J$, it also inherits some
desirable properties from its component manifolds.

\begin{prop}
Suppose that $\manifold_1, \manifold_2, \ldots \manifold_J$ are
isomorphic topological manifolds and $\manifold^*$ is defined as
above.
\begin{enumerate}
\item If $\manifold_1, \manifold_2, \ldots, \manifold_J$ are Riemannian,
then $\manifold^*$ is Riemannian.
\item If $\manifold_1, \manifold_2, \ldots, \manifold_J$ are compact,
then $\manifold^*$ is compact.
\end{enumerate}
\end{prop}
\begin{proof}
The proofs of these facts are straightforward and follow from the
fact that if the component manifolds are Riemannian or compact, then
the product manifold will be as well.  $\manifold^*$ then inherits
these properties as a submanifold of the product
manifold~\cite{Boothby}.
\end{proof}

Up to this point we have considered general topological manifolds.
In particular, we have {\em not} assumed that the component
manifolds are embedded in any particular space.  If each component
manifold $\manifold_j$ is embedded in $\reals^{N_j}$, the joint
manifold is naturally embedded in $\reals^{N^*}$ where $N^* =
\sum_{j=1}^J N_j$. Hence, the joint manifold can be viewed as a
model for data of {\em varying ambient dimension} linked by a common
parametrization.  In the sequel, we assume that each manifold
$\manifold_j$ is embedded in $\reals^N$, which implies that
$\manifold^* \subset \reals^{JN}$. Observe that while the intrinsic
dimension of the joint manifold remains constant at $K$, the ambient
dimension increases by a factor of $J$.  We now examine how a number
of geometric properties of the joint manifold compare to those of
the component manifolds.

We begin with the following simple observation that Euclidean
distances between points on the joint manifold are larger than
distances on the component manifolds.  In the remainder of this
paper, whenever we use the notation $\| \cdot \|$ we mean $\| \cdot
\|_{\ell_2}$, i.e., the $\ell_2$ (Euclidean) norm on $\reals^N$.
When we wish to differentiate this from other $\ell_p$ norms, we
will be explicit.

\begin{prop} \label{prop:euclidean}
Let $p = (p_1, p_2, \ldots, p_J)$ and $q = (q_1, q_2, \ldots, q_J)$
be two points on the joint manifold $\manifold^*$.  Then
$$
\|p-q\| = \sqrt{\sum_{j=1}^J \|p_j-q_j\|^2}.
$$
\end{prop}
\begin{proof}
This follows from the definition of the Euclidean norm:
$$
\|p-q\|^2 = \sum_{i=1}^{JN} (p(i)-q(i))^2 = \sum_{j=1}^J
\sum_{i=1}^N (p_j(i) - q_j(i))^2 = \sum_{j=1}^J \|p_j - q_j\|^2.
$$
\end{proof}

While Euclidean distances are important (especially when noise is
introduced), the natural measure of distance between a pair of
points on a Riemannian manifold is not Euclidean distance, but
rather the {\em geodesic distance}.  The geodesic distance between
points $p,q \in \manifold$ is defined as
\begin{equation} \label{eq:geodesic}
d_\manifold(p,q) = \inf \{ L(\gamma) : \gamma(0) = p, \gamma(1) = q
\},
\end{equation}
where $\gamma : [0,1] \rightarrow \manifold$ is a $C^1$-smooth curve
joining $p$ and $q$, and $L(\gamma)$ is the length of $\gamma$ as
measured by
\begin{equation}
L(\gamma) = \int_0^1 \|\dot{\gamma}(t)\| dt.
\end{equation}
In order to see how geodesic distances on $\manifold^*$ compare to
geodesic distances on the component manifolds, we will make use of
the following lemma.

\begin{lemma} \label{lemma:paths}
Suppose that $\manifold_1, \manifold_2, \ldots, \manifold_J$ are
Riemannian manifolds, and let $\gamma : [0,1] \rightarrow
\manifold^*$ be a $C^1$-smooth curve on the joint manifold.  Then we
can write $\gamma = (\gamma_1, \gamma_2, \ldots, \gamma_J)$ where
each $\gamma_j : [0,1] \rightarrow \manifold_j$ is a $C^1$-smooth
curve on $\manifold_j$, and
$$
\frac{1}{\sqrt{J}} \sum_{j=1}^J L(\gamma_j) \le L(\gamma) \le
\sum_{j=1}^J L(\gamma_j).
$$
\end{lemma}
\begin{proof}
We begin by observing that
\begin{equation} \label{eq:paths1}
L(\gamma) = \int_0^1 \|\dot{\gamma}(t)\| dt = \int_0^1 \sqrt{
\sum_{j=1}^J \|\dot{\gamma_j}(t)\|^2} \ dt.
\end{equation}
For a fixed $t$, let $x_j = \|\dot{\gamma_j}(t)\|$, and observe that
$(x_1, x_2, \ldots, x_J)$ is a vector in $\reals^J$.  Thus we may
apply the standard norm inequalities
\begin{equation} \label{eq:normiq}
\frac{1}{\sqrt{J}} \|x\|_{\ell_1} \le \|x\|_{\ell_2} \le
\|x\|_{\ell_1}
\end{equation}
to obtain
\begin{equation} \label{eq:paths2}
\frac{1}{\sqrt{J}} \sum_{j=1}^J \|\dot{\gamma_j}(t)\| \le \sqrt{
\sum_{j=1}^J \|\dot{\gamma_j}(t)\|^2} \le \sum_{j=1}^J
\|\dot{\gamma_j}(t)\|.
\end{equation}
Combining the right-hand side of (\ref{eq:paths2}) with
(\ref{eq:paths1}) we obtain
$$
L(\gamma) \le \int_0^1 \sum_{j=1}^J \|\dot{\gamma_j}(t)\| dt =
 \sum_{j=1}^J \int_0^1 \|\dot{\gamma_j}(t)\| dt = \sum_{j=1}^J
 L(\gamma_j).
$$
Similarly, from the left-hand side of (\ref{eq:paths2}) we obtain
$$
L(\gamma) \ge \int_0^1 \frac{1}{\sqrt{J}} \sum_{j=1}^J
\|\dot{\gamma_j}(t)\| dt = \frac{1}{\sqrt{J}}
 \sum_{j=1}^J \int_0^1 \|\dot{\gamma_j}(t)\| dt = \frac{1}{\sqrt{J}} \sum_{j=1}^J
 L(\gamma_j).
$$
\end{proof}

We are now in a position to compare geodesic distances on
$\manifold^*$ to those on the component manifold.

\begin{thm} \label{thm:geod}
Suppose that $\manifold_1, \manifold_2, \ldots, \manifold_J$ are
Riemannian manifolds. Let $p = (p_1, p_2, \ldots, p_J)$ and $q =
(q_1, q_2, \ldots, q_J)$ be two points on the corresponding joint manifold
$\manifold^*$.  Then
\begin{equation} \label{eq:geod1}
d_{\manifold^*}(p,q) \ge \frac{1}{\sqrt{J}} \sum_{j=1}^J
d_{\manifold_j}(p_j,q_j).
\end{equation}
If the mappings $\psi_2, \psi_3, \ldots, \psi_J$ are {\em
isometries}, i.e., $d_{\manifold_1}(p_1,q_1) =
d_{\manifold_j}(\psi_j(p_1),\psi_j(q_1))$ for any $j$ and for any
pair of points ($p,q$), then
\begin{equation} \label{eq:geod2}
d_{\manifold^*}(p,q) = \frac{1}{\sqrt{J}} \sum_{j=1}^J
d_{\manifold_j}(p_j,q_j) = \sqrt{J} \cdot d_{\manifold_1}(p_1,q_1).
\end{equation}
\end{thm}
\begin{proof}
If $\gamma$ is a geodesic path between
$p$ and $q$, then from Lemma \ref{lemma:paths},
$$
d_{\manifold^*}(p,q) = L(\gamma) \ge \frac{1}{\sqrt{J}} \sum_{j=1}^J
L(\gamma_j).
$$
By definition $L(\gamma_j) \ge d_{\manifold_j}(p_j,
q_j)$; hence, this establishes (\ref{eq:geod1}).

Now observe that lower bound in Lemma \ref{lemma:paths} is derived
from the lower inequality of (\ref{eq:normiq}).  This inequality is
attained with equality if and only if each term in the sum is equal,
i.e., $L(\gamma_j) = L(\gamma_k)$ for all $j$ and $k$. This is
precisely the case when $\psi_2, \psi_3, \ldots, \psi_J$ are
isometries.  Thus we obtain
$$
d_{\manifold^*}(p,q) = L(\gamma) = \frac{1}{\sqrt{J}} \sum_{j=1}^J
L(\gamma_j) = \sqrt{J} L(\gamma_1).
$$
We now conclude that $L(\gamma_1) = d_{\manifold_1}(p_1,q_1)$ since
if we could obtain a shorter path $\tilde{\gamma_1}$ from $p_1$ to
$q_1$ this would contradict the assumption that $\gamma$ is a
geodesic on $\manifold^*$, which establishes (\ref{eq:geod2}).
\end{proof}

Next, we study local smoothness and global self avoidance properties
of the joint manifold using the notion of {\em condition number}.

\begin{definition}{\em \cite{niyogi}}
Let $\manifold$ be a Riemannian submanifold of $\reals^\dim$. The
{\bf condition number} is defined as $1/\condition$, where
$\condition$ is the largest number satisfying the following: the
open normal bundle about $\manifold$ of radius $r$ is embedded in
$\reals^\dim$ for all $r < \condition$.
\end{definition}

The condition number of a given manifold controls both local
smoothness properties and global properties of the manifold.
Intuitively, as $1/\tau$ becomes smaller, the manifold becomes
smoother and more self-avoiding.  This is made more precise in the
following lemmata.

\begin{lemma}{\em \cite{niyogi}}
Suppose $\manifold$ has condition number $1/\tau$.  Let $p,q \in
\manifold$ be two distinct points on $\manifold$, and let
$\gamma(t)$ denote a unit speed parameterization of the geodesic
path joining $p$ and $q$.  Then
$$
\max_t \|\ddot{\gamma}(t)\| \le \frac{1}{\tau}.
$$
\end{lemma}

\begin{lemma}{\em \cite{niyogi}} \label{lem:cond_geodesics}
Suppose $\manifold$ has condition number $1/\tau$.  Let $p,q \in
\manifold$ be two points on $\manifold$ such that $\|p-q\| = d.$  If
$d \le \tau/2$, then the geodesic distance $d_\manifold(p,q)$ is
bounded by
$$
d_\manifold(p,q) \le \tau (1 - \sqrt{1 - 2d/\tau}).
$$
\end{lemma}

We wish to show that if the component manifolds are smooth and self
avoiding, the joint manifold is as well. It is not easy to prove
this in the most general case, where the only assumption is that
there exists a homeomorphism (i.e., a continuous bijective map
$\psi$) between every pair of manifolds. However, suppose the
manifolds are {\em diffeomorphic}, i.e., there exists a continuous
bijective map between tangent spaces at corresponding points on
every pair of manifolds. In that case, we make the following assertion.
\begin{thm} \label{thm:cond_jam}
Suppose that $\manifold_1, \manifold_2, \ldots, \manifold_J$ are
Riemannian submanifolds of $\reals^N$, and let $1/\tau_j$ denote the
condition number of $\manifold_j$. Suppose also that the $\psi_2,
\psi_3, \ldots, \psi_J$ that define the corresponding joint manifold
$\manifold^*$ are diffeomorphisms.  If $1/\tau^*$ is the condition
number of $\manifold^*$, then
$$
\tau^* \ge \mathop{\min}_{1\le j \le J} \tau_j,
$$
or equivalently,
$$
\frac{1}{\tau^*} \le \mathop{\max}_{1 \le j \le J} \frac{1}{\tau_j}.
$$
\end{thm}
\begin{proof}
Let $p \in \manifold^*$, which we can write as $p = (p_1, p_2,
\ldots, p_J)$ with $p_j \in \manifold_j$. Since the
$\{\psi_j\}_{j=2}^J$ are diffeomorphisms, we may view $\manifold^*$
as being diffeomorphic to $\manifold_1$; i.e., we can build a
diffeomorphic map from $\manifold_1$ to $\manifold^*$ as
$$
p = \psi^*(p_1) := (p_1, \psi_2(p_2), \ldots, \psi_J(p_J)).
$$

We also know that given any two manifolds linked by a diffeomorphism
$\psi_j: \manifold_1 \to \manifold_j$, each vector $v_1$ in the
tangent space $T_1(p_1)$ of the manifold $\manifold_1$ at the point
$p_1$ is {\em uniquely} mapped to a tangent vector $v_j :=
\phi_j(v_1)$ in the tangent space $T_j(p_j)$ of the manifold
$\manifold_j$ at the point $p_j = \psi_j(p_1)$ through the map
$\phi_j := \jacobian \circ \psi_j(p_1)$ , where $\jacobian$ denotes
the Jacobian operator.

Consider the application of this property to the diffeomorphic manifolds $\manifold_1$ and $\manifold^*$. In this case, the tangent vector $v_1 \in T_1(p_1)$ to the manifold $\manifold_1$ can be uniquely identified with a tangent vector $v = \phi^*(v_1) \in T^*(p)$ to the manifold $\manifold^*$. This mapping is expressed as
\begin{equation}
\phi^*(v_1) = \jacobian \circ \psi^*(p_1) =  (v_1, \jacobian \circ
\psi_2(p_1), \ldots, \jacobian \circ \psi_J(p_1)), \nonumber
\end{equation}
since the Jacobian operates componentwise. Therefore, the tangent vector $v$ can be written as
\begin{eqnarray}
v &=& \phi^*(v_1) = (v_1, \phi_2(v_1), \ldots, \phi_J(p_1)), \nonumber \\
&=& (v_1, v_2, \ldots, v_J).
\nonumber
\end{eqnarray}
In other words, a tangent vector to the joint manifold can be decomposed into $J$ component vectors, each of which are tangent to the corresponding component manifolds.

Using this fact, we now show that a vector $\eta$ that is normal to
$\manifold^*$ can also be broken down into sub-vectors that are
normal to the component manifolds. Consider $p \in \manifold^*$, and
denote $T^*(p)^\perp$ as the normal space at $p$. Suppose $\eta =
(\eta_1, \ldots, \eta_J) \in T^*(p)^\perp$. Decompose each $\eta_j$
as a projection onto the component tangent and normal spaces, i.e.,
for $j = 1,\ldots, J$,
$$
\eta_j = x_j + y_j, \qquad x_j \in T_j(p_j),~y_j \in T_j(p_j)^\perp .
$$
such that $\iprod{x_j}{y_j} = 0$ for each $j$. Let $x = (x_1,
\ldots, x_J)$ and $y = (y_1, \ldots, y_J)$. Then $\eta = x+y$, and
since $y$ is tangent to the joint manifold $\manifold^*$, we have
$\iprod{\eta}{y} = \iprod{x+y}{x} = 0$, and thus
$$
\iprod{y}{x}=-\|x\|^2.
$$
But,
$$
\iprod{y}{x} = \sum_{j=1}^J \iprod{y_j}{x_j} = 0 .
$$
Hence $x = 0$, i.e., each $\eta_j$ is normal to $\manifold_j$.

Armed with this last fact, our goal now is to show that if $r < \min_{1 \le j \le J} \tau_j$ then the normal
bundle of radius $r$ is embedded in $\reals^\dim$, or equivalently,
that $p + \eta \neq q + \nu$ provided that $\|\eta\|, \|\nu\| \le
r$. Indeed, suppose $\|\eta\|, \|\nu\| \le r < \min_{1\le j \le J} \tau_j$. Since
$\|\eta_j\| \le \|\eta\|$ and $\|\nu_j\| \le \|\nu\|$ for all $1 \le j \le J$,
we have that $\|\eta_j\|, \|\nu_j\| < \min_{1 \le i \le J} \tau_i \le \tau_j$. Since we have proved that $\eta_j, \nu_j$ are vectors in the normal bundle of $\manifold_j$ and their magnitudes are less than $\tau_j$, then $p_j + \eta_j \neq q_j + \nu_j$ by the definition of condition number. Thus $p + \eta \neq q + \nu$ and the result follows.
\end{proof}

This result states that for general manifolds, the most we can say
is that the condition number of the joint manifold is guaranteed to
be less than that of the {\em worst} manifold.  However, in practice
this is not likely to happen.  As an example, Figure
\ref{fig:normal_bundle} illustrates the point at which the normal
bundle intersects itself for the case of the joint manifold from
Figure \ref{fig:helix} (c).  In this case we obtain $\tau^*
= \sqrt{\pi^2/2 + 1}$.  Note that the condition numbers for the
manifolds $\manifold_1$ and $\manifold_2$ generating $\manifold^*$
are given by $\tau_1 = \infty$ and $\tau_2 = 1$.  Thus, while the
condition number in this case is not as good as the best manifold,
it is still notably better than the worst manifold.  In general,
even this example may be somewhat pessimistic, and it is possible
that in many cases the joint manifold may be better conditioned than
even the best manifold.

\begin{figure}[!t]
\centering
\includegraphics[width=0.5\hsize, viewport=50 60 400 300]{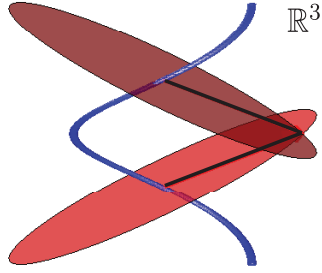}
\caption{\small\sl Point at which the normal bundle for the helix
manifold intersects itself.} \label{fig:normal_bundle}
\end{figure}

%% file: processing.tex
\section{Joint manifolds in signal processing}
\label{sec:processing}

Manifold models can be exploited by a number of algorithms for
signal processing tasks such as pattern classification, learning,
and control~\cite{multiscale}. The performance of such algorithms
often depends on geometric properties of the manifold model such as
its condition number and geodesic distances along its surface.  The
theory developed in Section~\ref{sec:jam} suggests that the joint
manifold preserves or improves these properties.  We will now see
that when noise is introduced these results suggest that, in the
case of multiple data sources, it can be extremely beneficial to use
algorithms specifically designed to exploit the joint manifold
structure.

\input{classify}
\input{learning}

%% file: classify.tex
\subsection{Classification} \label{sec:classify}

We first study the problem of manifold-based classification. The
problem is defined as follows: given manifolds $\manifold$ and $\manifoldN$, suppose
we observe a signal $y = x + n \in \reals^N$
where either $x \in \manifold$ or $x \in \manifoldN$ and $n$ is a
noise vector, and we wish to find a function $f: \reals^N \rightarrow
\{\manifold, \manifoldN \}$ that attempts to determine which
manifold ``generated'' $y$. We consider a simple classification
algorithm based on the {\em generalized maximum likelihood}
framework described in~\cite{smashedfilter}. The approach is to
classify by computing the distance from the observed signal $y$ to
each of the manifolds, and then classify based on which of these
distances is smallest, i.e., our classifier is
\begin{equation} \label{eq:classifier}
f(y) = \arg \min \left[ d(y,\manifold), d(y,\manifoldN) \right].
\end{equation}
We will measure the performance of this algorithm for a
particular pair of manifolds by considering the probability of
misclassifying a point from $\manifold$ as belonging to
$\manifoldN$, which we denote $P_{\manifold \manifoldN}$.

To analyze this problem, we employ three common notions of separation
in metric spaces:
\vspace{-2mm}
\begin{itemize}
\item The {\em minimum separation} distance
between two manifolds $\manifold$ and $\manifoldN$ is defined as
$$
\delta(\manifold,\manifoldN) = \inf_{p \in \manifold}
d(p,\manifoldN).
$$
\item The {\em Hausdorff distance} from $\manifold$ to $\manifoldN$
is defined to be
$$
D(\manifold,\manifoldN) = \sup_{p \in \manifold} d(p,\manifoldN),
$$
with $D(\manifoldN,\manifold)$ defined similarly. Note that
$\delta(\manifold,\manifoldN) = \delta(\manifoldN,\manifold)$, while
in general \\ $D(\manifold,\manifoldN) \neq
D(\manifoldN,\manifold)$.
\item The {\em maximum separation} distance between manifolds $\manifold$ and $\manifoldN$ is defined as
$$
\Delta(\manifold, \manifoldN) = \sup_{x \in \manifold} \sup_{y \in \manifoldN} \|x - y \|.
$$
\end{itemize}
\vspace{-2mm}
As one might expect, $P_{\manifold \manifoldN}$ is controlled by the
separation distances.  For example, suppose that $x \in \manifold$;
if the noise vector $n$ is bounded and satisfies $\|n\| <
\delta(\manifold,\manifoldN)/2$, then we have that $d(y,\manifold)
\le \|n\| < \delta(\manifold,\manifoldN)/2$ and hence
\begin{eqnarray*}
\delta(\manifold,\manifoldN) & = & \inf_{p \in \manifold, q \in
\manifoldN} \|p - q\| \\
& = & \inf_{p \in \manifold, q \in \manifoldN} \|p - y +y - q\| \\
 & \le & \inf_{p \in \manifold, q \in
\manifoldN} \|p - y\| +  \|y - q\| \\
 & = & d(y,\manifold) + d(y,\manifoldN) \\
 & < & \delta(\manifold,\manifoldN)/2 + d(y,\manifoldN).
\end{eqnarray*}
Thus we are guaranteed that
$$
d(y,\manifoldN) > \delta(\manifold,\manifoldN)/2 .
$$
Therefore, $d(y, \manifold) < d(y, \manifoldN)$ and the classifier
defined by (\ref{eq:classifier}) satisfies $P_{\manifold \manifoldN}
= 0$. We can refine this result in two possible ways. First, note
that the amount of noise $\epsilon$ that we can tolerate without
making an error depends on $x$. Specifically, for a given $x \in
\manifold$, provided that $\|n\| \le d(x,\manifoldN)/2$ we still
have that $P_{\manifold \manifoldN} = 0$.  Thus, for a given $x \in
\manifold$ we can tolerate noise bounded by $d(x,\manifoldN)/2 \in
[\delta(\manifold,\manifoldN)/2,D(\manifold,\manifoldN)/2]$.

A second possible refinement that we will explore below is to ignore this
dependence of $x$, but to extend our noise model to the case where
$\|n\| > \delta(\manifold,\manifoldN)/2$ with non-zero probability.
We can still bound $P_{\manifold \manifoldN}$ since
\begin{equation} \label{eq:pebound}
P_{\manifold \manifoldN} \le P(\|n\| >
\delta(\manifold,\manifoldN)/2).
\end{equation}

We provide bounds on this probability for both the component
manifolds and the joint manifold as follows: first, we first compare the
separation distances for these cases.
\begin{thm} \label{thm:djam}
Consider the joint manifolds $\manifold^* \subset \manifold_1 \times
\manifold_2 \times \cdots \times \manifold_J$ and $\manifoldN^*
\subset \manifoldN_1 \times \manifoldN_2 \times \cdots \times
\manifoldN_J$.  Then, the following bounds hold:
\begin{enumerate}
\item Joint minimum separation:
\begin{equation}
\sum_{j=1}^J \delta^2(\manifold_j,\manifoldN_j) \leq
\delta^2(\manifold^*,\manifoldN^*) \leq \min_{1 \leq k \leq J}
\left( \delta^2(\manifold_k, \manifoldN_k) + \sum_{j\neq k}
\Delta^2(\manifold_j,\manifoldN_j)\right). \label{eq:jms}
\end{equation}
\item Joint Hausdorff separation from $\manifold^*$ to $\manifoldN^*$:
\begin{equation}
\max_{1 \leq k \leq J} \left(D^2(\manifold_k,\manifoldN_k) + \sum_{j
\neq k} \delta^2(\manifold_j,\manifoldN_j)\right) \leq
D^2(\manifold^*,\manifoldN^*) \leq \sum_{j=1}^J
\Delta^2(\manifold_j,\manifoldN_j). \label{eq:jhs}
\end{equation}
\item Joint maximum separation from $\manifold^*$ to $\manifoldN^*$:
\begin{equation}
\max_{1 \leq k \leq J} \left(\Delta^2(\manifold_k,\manifoldN_k) +
\sum_{j \neq k} \delta^2(\manifold_j,\manifoldN_j)\right) \leq
\Delta^2(\manifold^*,\manifoldN^*) \leq \sum_{j=1}^J
\Delta^2(\manifold_j,\manifoldN_j). \label{eq:jmaxs}
\end{equation}
\end{enumerate}
\end{thm}
\begin{proof}
Inequality~(\ref{eq:jms}) is a simple corollary of Proposition
\ref{prop:euclidean}. Let $p = (p_1, p_2, \ldots, p_J)$ and $q =
(q_1, q_2, \ldots, q_J)$ respectively be the points on $\manifold^*$
and $\manifoldN^*$ for which the minimum separation distance
$\delta(\manifold^*,\manifoldN^*)$ is attained, i.e.,
$$
(p, q) = \arg \inf_{p \in \manifold^*} \inf_{q \in \manifoldN^*} \|p
- q\| .
$$
Then,
\begin{eqnarray*}
\delta^2(\manifold^*,\manifoldN^*)&=& \|p-q\|^2 = \sum_{j=1}^J \|p_j - q_j \|^2 \\
&\geq& \sum_{j=1}^J \delta^2(\manifold_j,\manifoldN_j),
\end{eqnarray*}
since the distance between two points in any given component space
is greater than the minimum separation distance corresponding to
that space.  This establishes the lower bound in~(\ref{eq:jms}). We
obtain the upper bound by selecting a $k$, and selecting $p \in
\manifold^*$ and $q \in \manifoldN^*$ such that $p_k$ and $q_k$
attain the minimum separation distance $\delta(\manifold_k,
\manifoldN_k)$.  From the definition of $\delta(\manifold^*,
\manifoldN^*)$, we have that
\begin{eqnarray*}
\delta^2(\manifold^*,\manifoldN^*)&\le& \|p-q\|^2 = \sum_{j=1}^J \|p_j - q_j \|^2 \\
&=& \delta^2(\manifold_k,\manifoldN_k) + \sum_{j \neq k} \|p_j -
q_j\|^2 \\
&\le & \delta^2(\manifold_k,\manifoldN_k) + \sum_{j \neq k}
\Delta^2(\manifold_j,\manifoldN_j),
\end{eqnarray*}
and since this holds for {\em every} choice of $k$, (\ref{eq:jms})
follows by taking the minimum over all $k$.

To prove inequality~(\ref{eq:jhs}), we follow a similar course. We
begin by selecting $p \in \manifold^*$ and $q \in \manifoldN^*$ that
satisfy
$$
(p, q) = \arg \sup_{p \in \manifold^*} \inf_{q \in \manifoldN^*} \|p
- q\|.
$$
Then,
\begin{eqnarray*}
D^2(\manifold^*,\manifoldN^*)&=& \|p-q\|^2 = \sum_{j=1}^J \|p_j - q_j \|^2 \\
&\leq& \sum_{j=1}^J \Delta^2(\manifold_j,\manifoldN_j),
\end{eqnarray*}
which establishes the upper bound in~(\ref{eq:jhs}).  To obtain the
lower bound, we again select a $k$, and now let $p \in \manifold^*$
be the point for which the corresponding at which the Hausdorff
separation {\em for the component manifold} $\manifold_k$ is
attained, i.e., the corresponding point $p_k$ is furthest away from
$\manifoldN_k$ as can be possible in $\manifold_k$. Let $q \in
\manifoldN^*$ be the nearest point in $\manifoldN^*$ to $p$. From
the definition of the Hausdorff distance, we get that
$$
D(\manifold^*,\manifoldN^*) \geq \|p-q\| ,
$$
since the Hausdorff distance is the {\em maximal} distance between the points in
$\manifold^*$ and their respective nearest neighbors in $\manifoldN^*$.
Again, it also follows that
\begin{eqnarray}
D^2(\manifold^*,\manifoldN^*) & \geq & \|p-q\|^2=\|p_k - q_k\|^2 + \sum_{j \neq k} \|p_j - q_j\|^2 \nonumber \\
&=&D^2(\manifold_k,\manifoldN_k) + \sum_{j \neq k} \|p_j - q_j\|^2 \nonumber \\
&\geq&D^2(\manifold_k,\manifoldN_k) + \sum_{j \neq k}
\delta^2(\manifold_j,\manifoldN_j) \nonumber .
\end{eqnarray}
Since this again holds for {\em every} choice of $k$,~(\ref{eq:jhs})
follows by taking the maximum over all $k$.

One can prove~(\ref{eq:jmaxs}) using the same technique used to
prove~(\ref{eq:jhs}).
\end{proof}

As an example, if we consider the case where the separation distances
are constant for all $j$, then the joint minimum separation distance
satisfies
\begin{eqnarray*}
\sqrt{J} \delta(\manifold_1,\manifoldN_1) \le
\delta(\manifold^*,\manifoldN^*) & \le &
\sqrt{\delta^2(\manifold_1,\manifoldN_1) +
(J-1)\Delta^2(\manifold_1,\manifoldN_1)} \\
&\le & \delta(\manifold_1,\manifoldN_1) + \sqrt{J-1}
\Delta(\manifold_1,\manifoldN_1)
\end{eqnarray*}
In the case where $\delta(\manifold_1,\manifoldN_1) \ll
\Delta(\manifold_1,\manifoldN_1)$ then we observe that
$\delta(\manifold^*,\manifoldN^*)$ can be considerably larger than
$\sqrt{J} \delta(\manifold_1,\manifoldN_1)$.  This means that we can
potentially tolerate much more noise while ensuring $P_{\manifold^*
\manifoldN^*} = 0$.  To see this, write $n = (n_1, n_2, \ldots,
n_J)$ and recall that we require $\|n_j\| < \epsilon =
\delta(\manifold_j,\manifoldN_j)/2$ to ensure that $P_{\manifold_j
\manifoldN_j} = 0$.  Thus, if we require that $P_{\manifold_j
\manifoldN_j} = 0$ for all $j$, then we have that
$$
\|n\| = \sqrt{\sum_{j=1}^J \|n_j\|^2 } < \sqrt{J} \epsilon =
\sqrt{J} \delta(\manifold_1,\manifoldN_1)/2.
$$
However, if we instead only require that $P_{\manifold^*
\manifoldN^*} = 0$ we only need $\| n\| <
\delta(\manifold^*,\manifoldN^*)/2$, which can be a significantly
less stringent requirement.

The benefit of classification using the joint manifold is made more
apparent when we extend our noise model to the case where we allow
$\|n_j\| > \delta(\manifold_j,\manifoldN_j)/2$ with non-zero
probability and apply (\ref{eq:pebound}).  To bound the probability
in (\ref{eq:pebound}), we will make use of the following adaptation
of Hoeffding's inequality~\cite{hoeffding}.
\begin{lemma} \label{lem:hoeff}
Suppose that $n_j \in \reals^N$ is a random vector that satisfies
$\|n_j\| \le \epsilon$, for $j = 1,2,\ldots,J$.  Suppose also that
the $n_j$ are independent and identically distributed (i.i.d.) with
$E[\|n_j\|] = \sigma$.  Then if $n = (n_1, n_2, \ldots, n_J) \in
\reals^{JN}$, we have that for any $\lambda > 0$,
$$
P\left( \|n\|^2 > J (\sigma^2 + \lambda) \right) \le \exp \left( -
\frac{2J\lambda^2}{\epsilon^4} \right).
$$
\end{lemma}

Using this lemma we can relax the assumption on $\epsilon$ so that
we only require that it is finite, and instead make the weaker
assumption that $E[\|n\|] = \sqrt{J} \sigma \le \delta(\manifold,
\manifoldN)/2$ for a particular pair of manifolds $\manifold$,
$\manifoldN$. This assumption ensures that $\lambda =
\delta^2(\manifold,\manifoldN)/4 - \sigma^2 > 0$, so that we can
combine Lemma \ref{lem:hoeff} with (\ref{eq:pebound}) to obtain a
bound on $P_{\manifold \manifoldN}$.  Note that if this condition
does not hold, then this is a very difficult classification problem
since the {\em expected} norm of the noise is large enough to push
us closer to the other manifold, in which case the simple classifier
given by (\ref{eq:classifier}) makes little sense.

We now illustrate how Lemma \ref{lem:hoeff} can be be used to
compare error bounds between classification using a joint manifold
and classification using a particular pair of component manifolds
$\manifold_k, \manifoldN_k$.

\begin{thm} \label{thm:bounds}
Suppose that we observe a vector $y = x + n$ where $x \in
\manifold^*$ and  $n = (n_1, n_2, \ldots, n_J)$ is a random vector
such that $\|n_j\| \le \epsilon$, for $j = 1,2,\ldots,J$, and that
the $n_j$ are i.i.d.\ with
$E[\|n_j\|] = \sigma \le \delta(\manifold_k,\manifoldN_k)/2$.  If
\begin{equation}
\delta(\manifold_k,\manifoldN_k) \le
\frac{\delta(\manifold^*,\manifoldN^*)}{\sqrt{J}},
\end{equation}
and we classify the observation $y$ according to (\ref{eq:classifier}), then
\begin{equation} \label{eq:pebound2}
P_{\manifold^* \manifoldN^*} \le \exp \left(-\frac{2c^*}{\epsilon^4}
\right),
\end{equation}
and
\begin{equation} \label{eq:pebound3}
P_{\manifold_k \manifoldN_k} \le \exp \left(-\frac{2c_k}{\epsilon^4}
\right),
\end{equation}
such that
$$
c^* > c_k.
$$
\end{thm}
\begin{proof}
First, observe that
\begin{equation} \label{eq:sigmabound}
\frac{\delta^2(\manifold^*,\manifoldN^*)}{J} \ge
\delta^2(\manifold_k,\manifoldN_k) \ge 4 \sigma^2.
\end{equation}
Thus, we may set $\lambda = \delta^2(\manifold^*,\manifoldN^*)/4J -
\sigma^2 > 0$ and apply Lemma \ref{lem:hoeff} to obtain
(\ref{eq:pebound2}) with
$$
c^* = J \left( \frac{\delta^2(\manifold^*,\manifoldN^*)}{4J} -
\sigma^2 \right)^2.
$$
Similarly, we may again apply Lemma \ref{lem:hoeff} by setting
$\lambda = \delta^2(\manifold_j,\manifoldN_j)/4 - \sigma^2 > 0$ and
$J=1$ to obtain (\ref{eq:pebound3}) with
$$
c_k = \left( \frac{\delta^2(\manifold_k,\manifoldN_k)}{4} - \sigma^2
\right)^2.
$$

It remains to show that $c^* > c_k$.  Thus, observe that
\begin{eqnarray*}
\delta^2(\manifold_k,\manifoldN_k) & \le &
\frac{\delta^2(\manifold^*,\manifoldN^*)}{J} \\
 & = & \frac{ \sqrt{J} \delta^2(\manifold^*,\manifoldN^*) -
 (\sqrt{J}-1) \delta^2(\manifold^*,\manifoldN^*)}{J} \\
 & = & \frac{\delta^2(\manifold^*,\manifoldN^*)}{\sqrt{J}} -
 (\sqrt{J}-1) \frac{\delta^2(\manifold^*,\manifoldN^*)}{J} \\
 & \le &\frac{\delta^2(\manifold^*,\manifoldN^*)}{\sqrt{J}} -
 4 \sigma^2 (\sqrt{J}-1),
\end{eqnarray*}
where the last inequality follows from (\ref{eq:sigmabound}).
Rearranging terms, we obtain
$$
\frac{\delta^2(\manifold_k,\manifoldN_k)}{4} - \sigma^2 \le \sqrt{J}
\left( \frac{ \delta^2(\manifold^*,\manifoldN^*)}{4 J} - \sigma^2
\right).
$$
Thus,
$$
\sqrt{c_k} \le \sqrt{c^*},
$$
and since $c_k > 0$ by assumption, we obtain
$$
c_k \le c^*,
$$
as desired.
\end{proof}
This result can be weakened slightly to obtain the following
corollary.
\begin{cor} \label{cor:bounds}
Suppose that we observe a vector $y = x + n$ where $x \in
\manifold^*$ and  $n = (n_1, n_2, \ldots, n_J)$ is a random vector
such that $\|n_j\| \le \epsilon$, for $j = 1,2,\ldots,J$ and that
the $n_j$ are i.i.d.\ with
$E[\|n_j\|] = \sigma \le \delta(\manifold_k,\manifoldN_k)/2$.  If
\begin{equation} \label{eq:cor_cond}
\delta^2(\manifold_k,\manifoldN_k) \le \frac{\sum_{j \neq k}
\delta^2(\manifold_j,\manifoldN_j)}{J-1},
\end{equation}
and we classify according to (\ref{eq:classifier}), then
(\ref{eq:pebound2}) and (\ref{eq:pebound3}) hold with the same
constants as in Theorem \ref{thm:bounds}.
\end{cor}
\begin{proof}
We can rewrite (\ref{eq:cor_cond}) as
$$
\delta^2(\manifold_k,\manifoldN_k) \le \frac{\sum_{j=1}^J
\delta^2(\manifold_j,\manifoldN_j) -
\delta^2(\manifold_k,\manifoldN_k)}{J-1}.
$$
After rearranging terms, this reduces to
$$
\delta^2(\manifold_k,\manifoldN_k) \le \frac{\sum_{j=1}^J
\delta^2(\manifold_j,\manifoldN_j)}{J}.
$$
Applying (\ref{eq:jms}) from Theorem \ref{thm:djam}, we obtain
$$
\delta^2(\manifold_k,\manifoldN_k) \le
\frac{\delta^2(\manifold^*,\manifoldN^*)}{J},
$$
which allows us to apply Theorem \ref{thm:bounds} to prove the
desired result.
\end{proof}

Corollary~\ref{cor:bounds} shows that we can expect joint
classification to outperform the $k$-th individual classifier
whenever the squared separation distance for the $k$-th component
manifolds is not too much larger than the average squared separation
distance among the remaining component manifolds.  Thus, we can
expect that the joint classifier is outperforming most of the
individual classifiers, but it is still possible that some of the
individual classifiers are doing better.  Of course, if one were
able to know in advance which classifiers were best, then one would
only use data from the best sensors.  We expect that a more typical
situation is when the separation distances are (approximately) equal
across all sensors, in which case the condition
in~(\ref{eq:cor_cond}) is true for all of the component manifolds.

%% file: learning.tex
\subsection{Manifold learning}
\label{sec:learning}

In contrast to the classification scenario described above, where we
knew the manifold structure {\em a priori}, we now consider manifold
{\em learning} algorithms that attempt to learn the manifold
structure by constructing a (possibly nonlinear) embedding of a
given point cloud into a subset of $\reals^L$, where $L < N$.
Typically, $L$ is set to $K$, the intrinsic manifold dimension.
Several such algorithms have been proposed, each giving rise to a
nonlinear map with its own special properties and advantages (e.g.
Isomap~\cite{isomap}, Locally Linear Embedding (LLE)~\cite{lle},
Hessian Eigenmaps~\cite{hlle}, etc.) Such algorithms provide a
powerful framework for navigation, visualization and interpolation
of high-dimensional data. For instance, manifold learning can be
employed in the inference of articulation parameters (eg., 3-D pose)
of points sampled from an image appearance manifold.

In particular, the Isomap algorithm deserves special mention. It assumes
that the point cloud consists of samples from a data
manifold that is (at least approximately) isometric to a convex
subset of Euclidean space. In this case, there exists an isometric
mapping $f$ from a parameter space $\Theta \subseteq \reals^K$
to the manifold $\manifold$ such that the geodesic distance between
every pair of data points is equal to the $\ell_2$ distance between their
corresponding pre-images in $\Theta$. In essence, Isomap attempts to
discover the coordinate structure of that $K$-dimensional space.

Isomap works in three stages:
\vspace{-2mm}
\begin{itemize}
\item We construct a graph $G$ that contains one vertex for each input data
point;  an edge connects two vertices if the Euclidean distance
between the corresponding data points is below a specified
threshold.
\item We weight each edge in the graph $G$ by computing the Euclidean distance between the corresponding data points.  We then estimate the geodesic distance between each pair
of vertices as the length of the shortest path between the
corresponding vertices in the graph $G$.
\item We embed the points in $\reals^K$ using multidimensional
scaling (MDS), which attempts to embed the points so that their
Euclidean distance approximates the geodesic distances estimated in
the previous step.
\end{itemize}
\vspace{-2mm}
A crucial component of the MDS algorithm is a suitable linear
transformation of the matrix of squared geodesic distances $D$; the
rank-$K$ approximation of this new matrix yields the best possible
$K$-dimensional coordinate structure of the input sample points in a
mean-squared sense. Further results on the performance of Isomap in
terms of geometric properties of the underlying manifold can be
found in~\cite{bernstein}.

We examine the performance of manifold learning using Isomap with
samples of the {\em joint manifold}, as compared to learning any of the
component manifolds.  We first assume that we are given noiseless
samples from the $J$ isometric component manifolds
$\manifold_1,\manifold_2,\ldots,\manifold_J$.  In order to judge the
quality of the embedding learned by the Isomap algorithm, we will
observe that for any pair of points $p,q$ on a manifold $\manifold$,
we have that
\begin{equation} \label{eq:isomap_iq}
\rho \le \frac{ \|p-q\| }{d_\manifold(p,q)} \le 1
\end{equation}
for some $\rho \in [0,1]$ that will depend on $p, q$.  Isomap will
perform well if the largest value of $\rho$ that
satisfies~(\ref{eq:isomap_iq}) for any pair of samples  that are
connected by an edge in the graph $G$ is close to $1$.  Using this
result, we can compare the performance of manifold learning using
Isomap on samples from the joint manifold $\manifold^*$ to using
Isomap on samples from a particular component manifold
$\manifold_k$.

\begin{thm} \label{thm:isomap_geo}
Let $\manifold^*$ be a joint manifold from $J$ isometric component
manifolds. Let $p = (p_1, p_2, \ldots, p_J)$ and $q = (q_1, q_2,
\ldots, q_J)$ denote a pair of samples of $\manifold^*$ and suppose
that we are given a graph $G$ that contains one vertex for each
sample. For each $k=1,\ldots,J$, define $\rho_j$ as the largest
value such that
\begin{equation} \label{eq:geothm1}
\rho_j \le \frac{ \|p_j-q_j\| }{d_{\manifold_j}(p_j,q_j)} \le 1
\end{equation}
for all pairs of points connected by an edge in $G$.  Then we have
that
\begin{equation} \label{eq:geothm2}
\sqrt{\frac{\sum_{j=1}^J \rho_j^2}{J}} \le \frac{ \|p-q\|
}{d_{\manifold^*}(p,q)} \le 1.
\end{equation}
\end{thm}
\begin{proof}
By Proposition~\ref{prop:euclidean},
$$
\|p-q\|^2 = \sum_{j=1}^J \|p_j-q_j\|^2,
$$
and from Theorem~\ref{thm:geod} we have that
$$
d_{\manifold^*}^2(p,q) = J d_{\manifold_1}^2(p_1,q_1).
$$
Thus,
\begin{eqnarray*}
\frac{\|p-q\|^2}{d_{\manifold^*}^2(p,q)} & = & \frac{\sum_{j=1}^J
\|p_j-q_j\|^2}{J d_{\manifold_1}^2(p_1,q_1)} \\
& = & \frac{1}{J} \sum_{j=1}^J \frac{\|p_j-q_j\|^2}{
d_{\manifold_1}^2(p_1,q_1)} \\
& = & \frac{1}{J} \sum_{j=1}^J \frac{\|p_j-q_j\|^2}{
d_{\manifold_j}^2(p_j,q_j)} \\
& \ge & \frac{1}{J} \sum_{j=1}^J \rho_j^2,
\end{eqnarray*}
which establishes the lower bound in (\ref{eq:geothm2}). The upper
bound is trivial since we always have that $d_{\manifold^*}(p,q) \ge
\|p-q\|$.
\end{proof}

From Theorem~\ref{thm:isomap_geo} we see that, in many cases, the
joint manifold estimates of the geodesic distances will be more
accurate than the estimates obtained using one of the component
manifolds. For instance, if for particular component
manifold $\manifold_k$ we observe that
$$
\rho_k \le \sqrt{\frac{\sum_{j=1}^J \rho_j^2}{J}},
$$
then we know that the joint manifold leads to better estimates.
Essentially, we can expect that the joint manifold will lead to
estimates that are better than the average case across the component
manifolds.

We now consider the case where we have a sufficiently dense sampling
of the manifolds so that the $\rho_j$ are very close to one, and
examine the case where we are obtaining noisy samples.  We will
assume that the noise affecting the data samples is i.i.d., and
demonstrate that any distance calculation performed on $\manifold^*$
serves as a better estimator of the pairwise (and consequently,
geodesic) distances between two points labeled by $p$ and $q$ than
that performed on any component manifold between their corresponding
points $p_j$ and $q_j$.

\begin{thm}
Let $\manifold^*$ be a joint manifold from $J$ isometric component
manifolds. Let $p = (p_1, p_2, \ldots, p_J)$ and $q = (q_1, q_2,
\ldots, q_J)$ be samples of $\manifold^*$ and assume that
$\|p_j-q_j\| = d$ for all $j$. Assume that we acquire noisy
observations $s = p + n$ and $r = q + n'$, where $n = (n_1, n_2,
\ldots, n_J)$ and $n' = (n'_1, n'_2, \ldots, n'_J)$ are independent
noise vectors with the same variance and norm bound
$$
\mathbb{E}[\|n_j\|^2] = \sigma^2~~\textrm{and}~~\|n_j\|^2 \le \epsilon,~~j = 1,\ldots,J.
$$
Then,
%$$P(\left|\|s-r\|^2/J-\|s-r\|^2\right| > \mu) \le c^J.$$
$$P\left(1-\delta \le \frac{\|s-r\|^2}{\|p-q\|^2 + 2J\sigma^2} \le 1+\delta\right) \ge 1-2c^{-{J^2}},$$
where $c = \mathrm{exp}\left(2 \delta^2 \left(\frac{d^2 + 2\sigma^2}{d\sqrt{\epsilon}+\epsilon}\right)^2\right)$.
% $c = \mathrm{exp}\left(-2\left(\frac{\delta\|s-r\|^2-2J\sigma^2}{\|s-r\|\sqrt{J\epsilon}+J\epsilon}\right)^2\right)$
\label{thm:jml}
\end{thm}
\begin{proof}We write the distance between the noisy samples as
\begin{equation} \label{eq:dists}
\|s-r\|^2 = \sum_{j=1}^J \lbrace {\|p_j-q_j\|^2 + 2 \langle p_j - q_j, n_j - n'_j\rangle + \|n_j - n'_j\|^2} \rbrace.
\end{equation}
%The isometric assumption on the component manifolds implies that
%that $\|p_j-q_j\|$ is constant for $j = 1,2, \ldots, J$; call this constant $d$.
%The quantity $2 (p_j - q_j)\cdot(n_j - n'_j)$ is a scalar Gaussian random variable
%with zero mean and variance $8 \|p_j-q_j\| \sigma^2$, while
%$\sum_{j=1}^J \|n_j - n'_j\|^2$ is a $\chi$-squared random variable with mean
%$2 J \sigma^2$ and variance $4 J \sigma^4$. Overall, $\|s-r\|^2$ is a mixture
%of (non-independent) Gaussian and $\chi$-squared distributions, with mean
%$J d^2 + 2 J \sigma^2$.
This can be rewritten as
\begin{equation} \label{eq:dists2}
\|s-r\|^2 - \|p-q\|^2 = \sum_{j=1}^J \lbrace {2 \langle p_j - q_j, n_j - n'_j\rangle + \|n_j - n'_j\|^2} \rbrace.
\end{equation}
We obtain the following statistics for the term inside the sum:
\begin{eqnarray}
\mathbb{E}[\langle p_j-q_j,n_j-n'_j \rangle + \|n_j-n'_j\|^2] &=& 2\sigma^2, \nonumber \\
\left|\langle p_j-q_j,n_j-n'_j \rangle + \|n_j-n'_j\|^2\right| &\le & 2d\sqrt{\epsilon}+\epsilon. \nonumber
\end{eqnarray}
Using Hoeffding's inequality, we obtain
$$P\left(\left|\sum_{j=1}^J \lbrace {2\langle p_j-q_j,n_j-n'_j \rangle + \|n_j-n'_j\|^2}\rbrace - 2J \sigma^2 \right| > J^2 \lambda\right) \le 2e^{-\frac{2J\lambda^2}{(2d\sqrt{\epsilon}+\epsilon)^2}}.$$
This result is rewritten to obtain
\begin{eqnarray}
P\left(\left|\|s-r\|^2 - \|p-q\|^2 - 2J\sigma^2 \right| > J\lambda\right) &\le& 2e^{-\frac{2J^2\lambda^2}{(2d\sqrt{\epsilon}+\epsilon)^2}}, \nonumber \\
P\left(\left|\|s-r\|^2 - \|p-q\|^2 - 2J\sigma^2\right| \le J\lambda\right) &\ge& 1-2e^{-\frac{2J^2\lambda^2}{(2d\sqrt{\epsilon}+\epsilon)^2}}. \nonumber
\end{eqnarray}
Simplifying, we get
$$
P\left(1 - \frac{\lambda}{d^2 + 2 \sigma^2} \le \frac{\|s-r\|^2}{\|p-q\|^2 + 2J\sigma^2} \le 1 + \frac{\lambda}{d^2 + 2\sigma^2}\right) \ge 1-2e^{-\frac{2J^2\lambda^2}{(2d\sqrt{\epsilon}+\epsilon)^2}}.
$$
Replace $\delta = \frac{\lambda}{d^2 + 2\sigma^2}$ to obtain the
result.
\end{proof}
%The exact statistical variance of this random variable is hard to compute, but
%the key point to note is that {\em both} the mean and the variance of this
%distribution scale linearly with $J$. Hence, if we normalize our distance
%estimate by the number of component manifolds, we obtain an estimate for
%the distance:
%$$
%\widehat{d^2} = d^2 + 2 \sigma^2 + n'' ,
%$$
%where $n''$ is a zero mean random variable whose variance scales linearly
%with $\sigma^4$ and inversely with $J$.

We observe that the estimate of the true distance suffers from a
constant small bias; this can be handled using a simple debiasing
step.\footnote{Manifold learning algorithms such as Isomap deal with
biased estimates of distances by ``centering'' the matrix of squared
distances, i.e., removing the mean of each row/column from every
element.} Theorem~\ref{thm:jml} indicates that the probability of
large deviations in the estimated distance decreases {\em
exponentially} in the number of component manifolds $J$; thus the
``denoising'' effect in joint manifold learning is manifested even
in the case where only a small number of component manifolds are
present.

As an example, we consider three different manifolds formed by images of
an ellipse with major axis $a$ and minor axis $b$ translating in a
2-D plane; an example point is shown in Figure~\ref{fig:ellips}.
The eccentricity of the ellipse directly affects the condition number $1/ \tau$
of the image articulation manifold; in fact, it can be shown that articulation
manifolds formed by more eccentric ellipses exhibit higher values for the
condition number. Consequently, we expect that it is ``harder'' to learn such
manifolds.

\begin{figure}[t]
\centering
\begin{tabular}{ccc}
{\includegraphics[width=0.25\hsize]{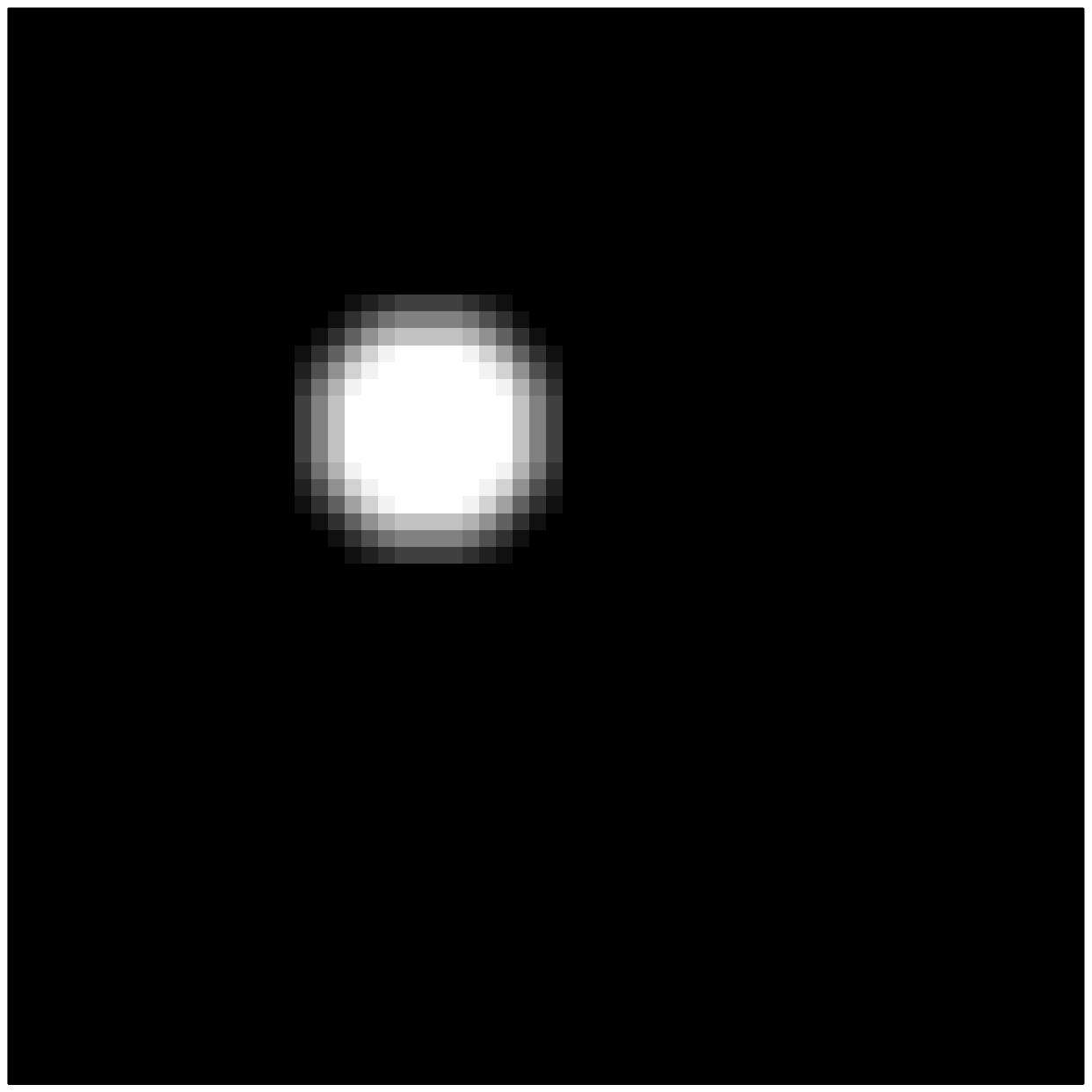}} &
{\includegraphics[width=0.25\hsize]{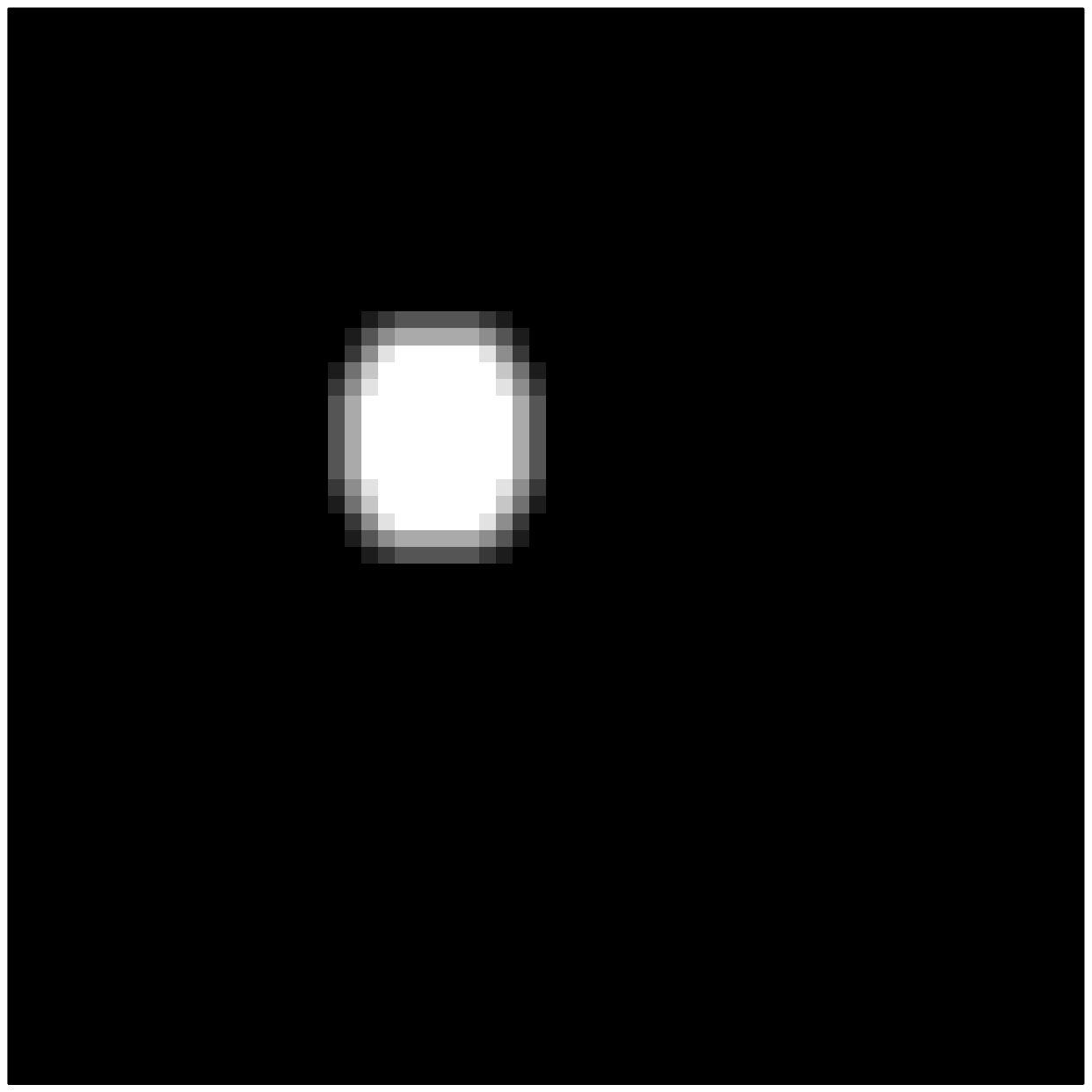}} &
{\includegraphics[width=0.25\hsize]{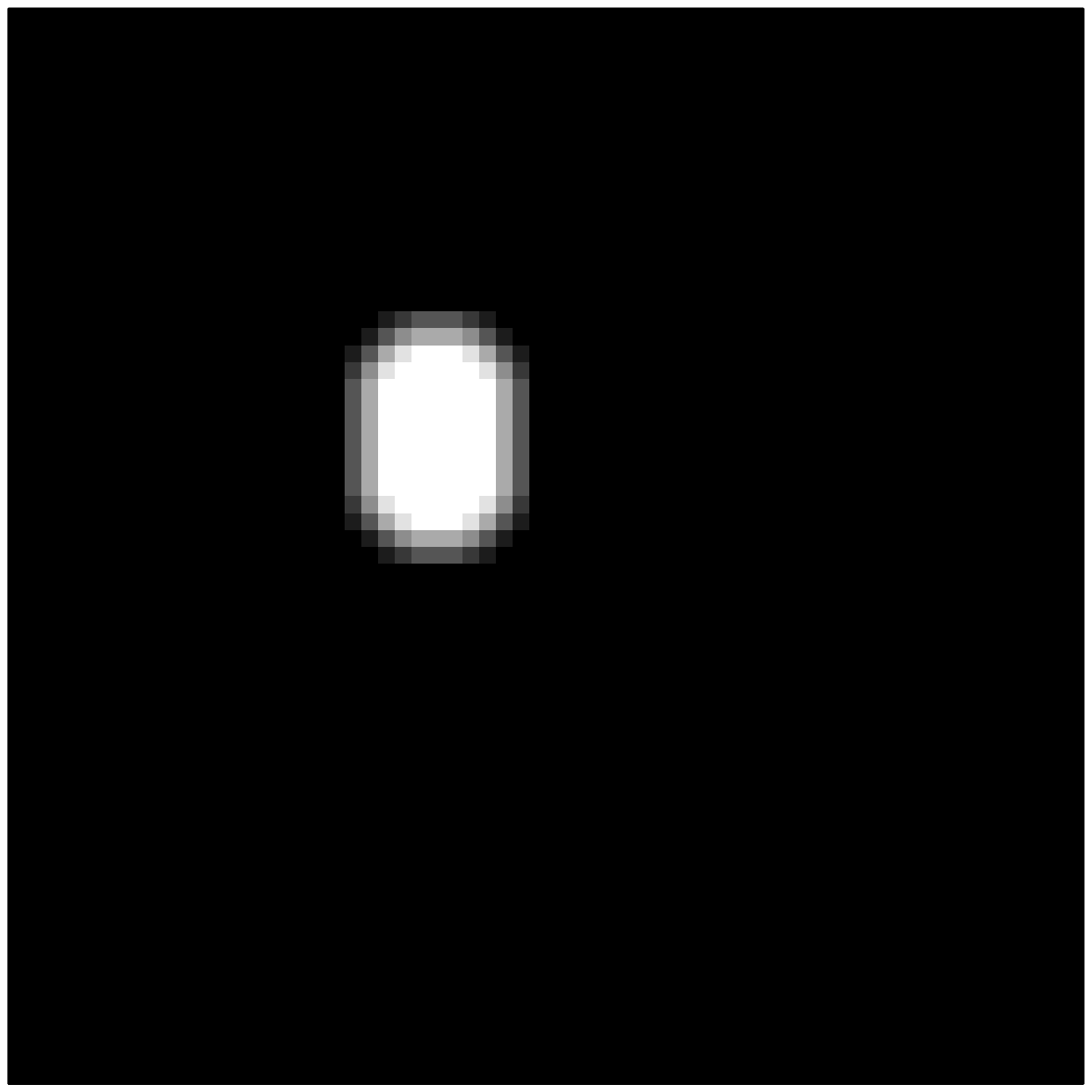}} \\
(i) $(a,b) = (7,7)$ & (ii) $(a,b) = (7,6)$  & (iii) $(a,b) = (7,5)$
\end{tabular}
\caption{\small\sl Three articulation manifolds embedded in $\reals^{4096}$ sharing
a common 2-D parameter space $\Theta$.}
\label{fig:ellips}
\end{figure}

Figure~\ref{fig:ellipselearn} shows that this is indeed the case. We add a small
amount of white gaussian noise to each image and apply the Isomap
algorithm~\cite{isomap} to both the individual datasets as well as the
concatenated dataset. We observe that the 2-D embedding is poorly learnt in
each of the individual manifolds, but improves visibly when the data ensemble
is modeled using a joint manifold.

\begin{figure}[t]
\centering
\begin{tabular}{cc}
{\includegraphics[width=0.3\hsize]{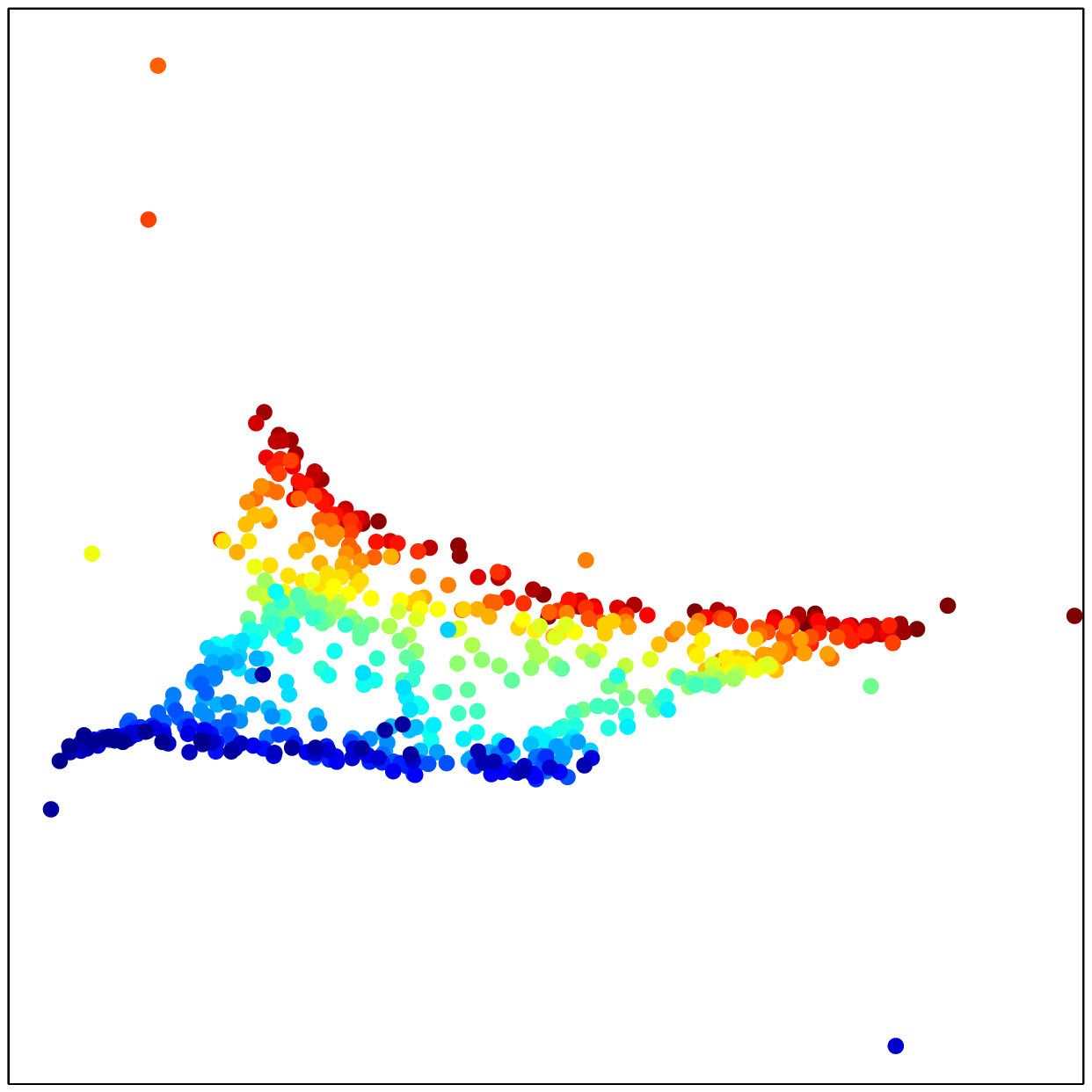}} &
{\includegraphics[width=0.3\hsize]{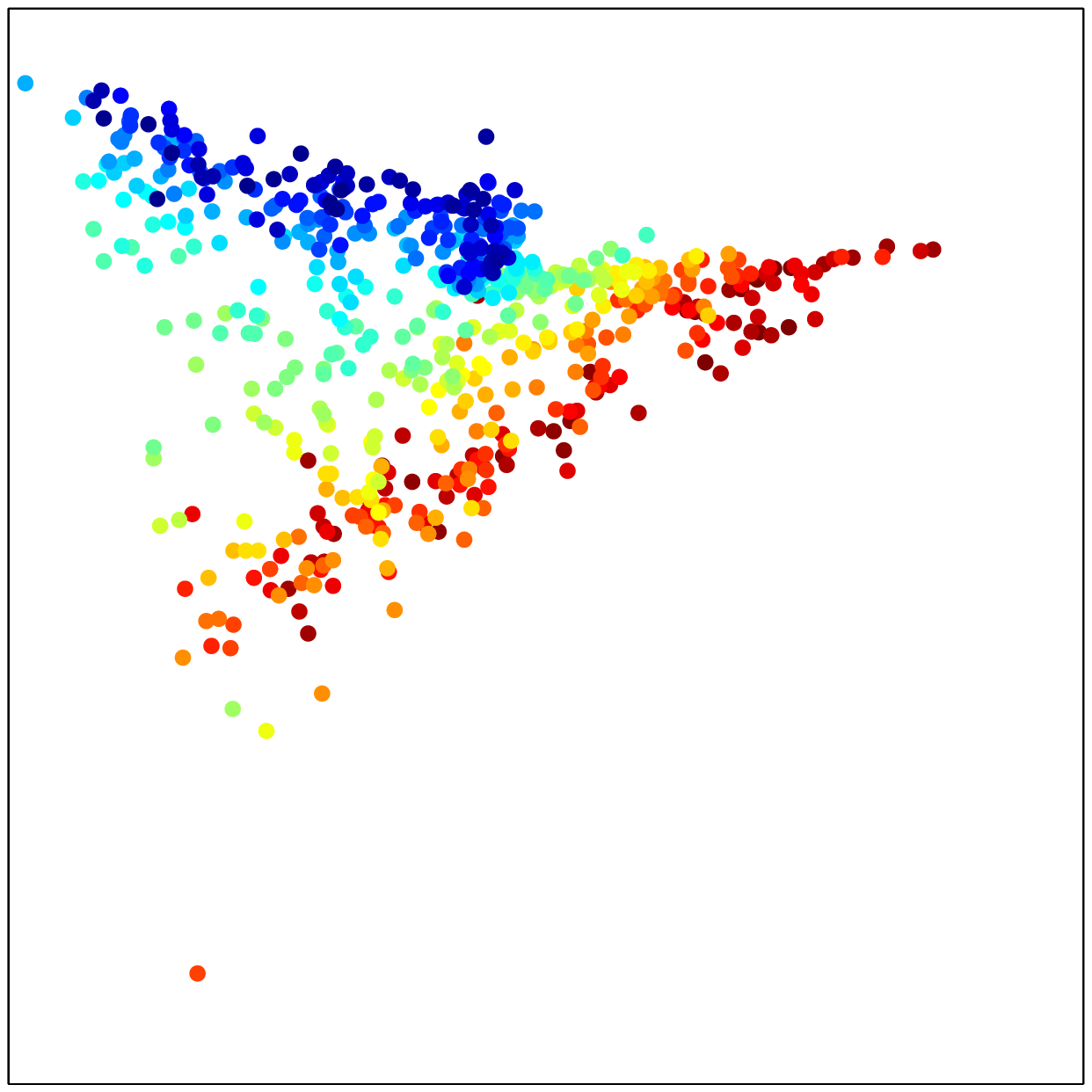}} \\
(i) & (ii) \\
{\includegraphics[width=0.3\hsize]{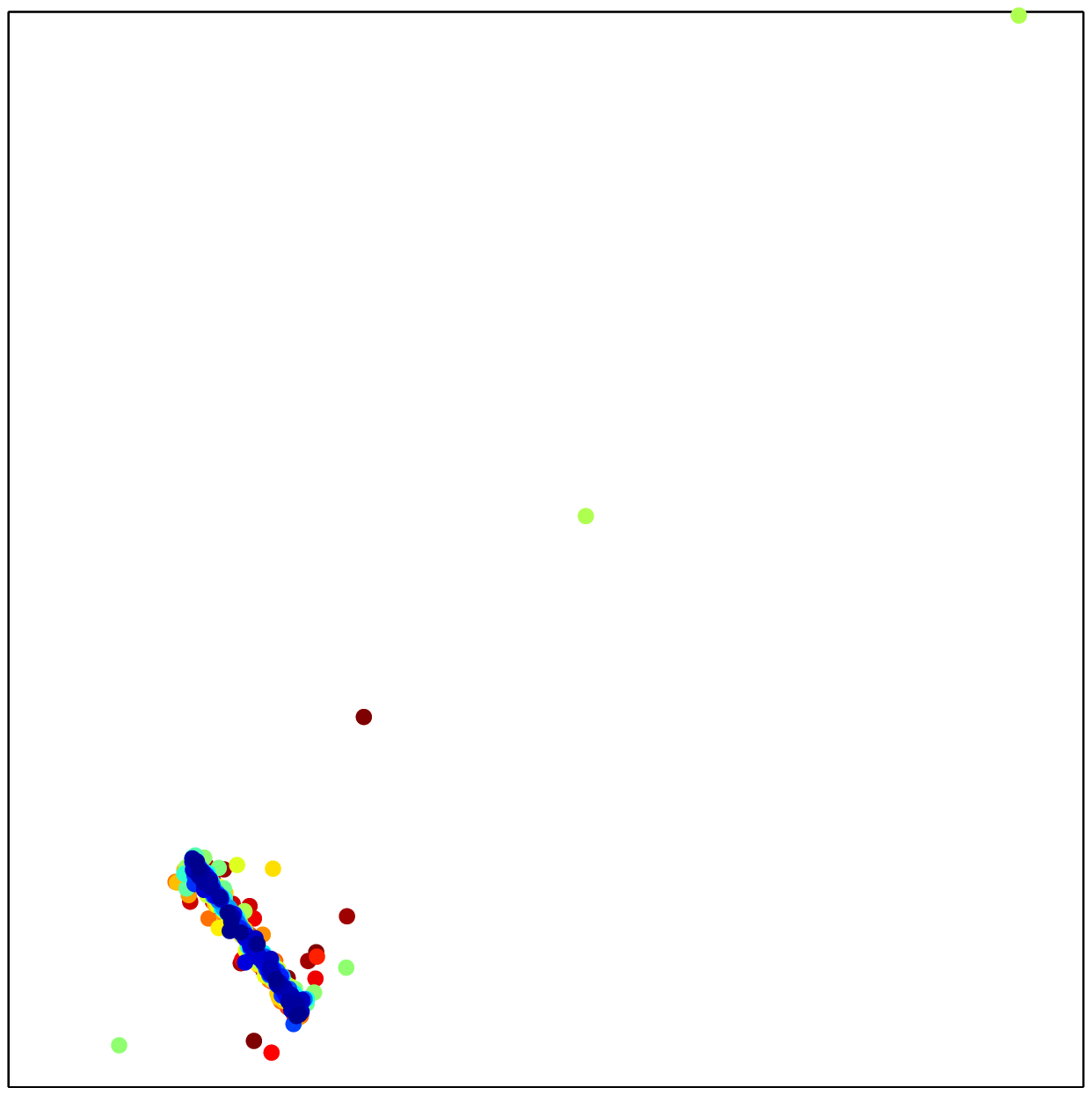}} &
{\includegraphics[width=0.3\hsize]{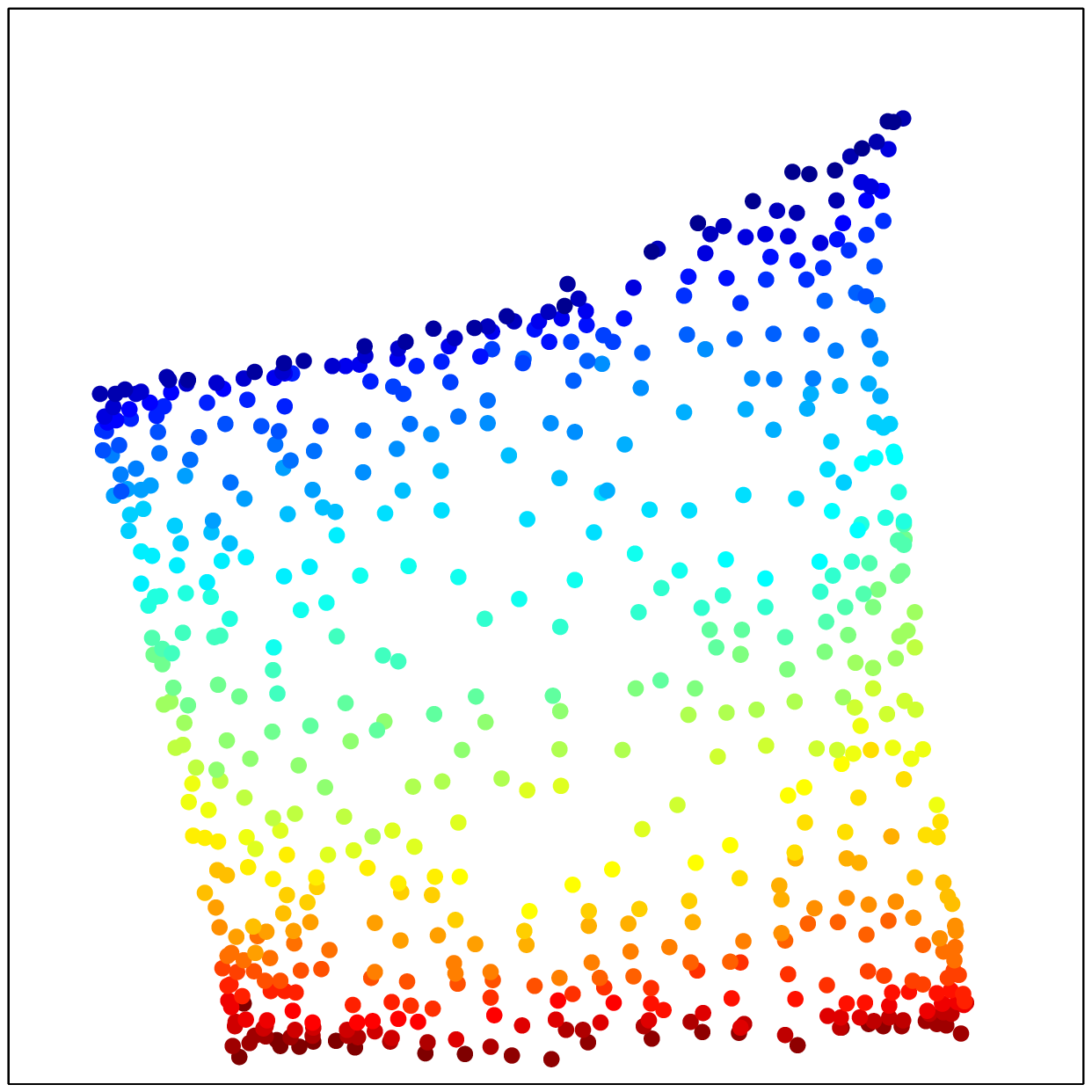}} \\
(iii) & Joint manifold
\end{tabular}
\caption{\small\sl Results of Isomap applied to the translating ellipse image data sets.}
\label{fig:ellipselearn}
\end{figure}

%% file: randproj.tex
\section{Joint manifolds for efficient dimensionality reduction} \label{sec:defjam}

We have shown that joint manifold models for data ensembles can significantly 
improve the performance on a variety of signal processing tasks, where 
performance is quantified using metrics like probability of error for detection and 
accuracy for parameter estimation and manifold learning.  In particular, we have 
observed that performance tends to improve exponentially fast as we increase
the number of component manifolds $J$.  However, we have ignored
that when $J$ and the ambient dimension of the manifolds $N$ become
large, the dimensionality of the joint manifold --- $JN$
--- may be so large that it becomes impossible to perform any meaningful
computations. Fortunately, we can transform the data into a more
amenable form via the method of {\em random projections}: it has been shown
that the essential structure of a $K$-dimensional manifold with
condition number $1 / \tau$ residing in $\reals^N$ is approximately
preserved under an orthogonal projection into a random subspace of
dimension $O(K \log (N/ \tau)) \ll N$~\cite{BaraniukWakin06}.
This result can be leveraged to
enable efficient design of inference applications, such as
classification using multiscale navigation~\cite{SmashedFilterICIP},
intrinsic dimension estimation, and manifold
learning~\cite{chmwrbnips}.

We can apply this result individually for each sensor acquiring
manifold-modeled data.  Suppose $N$-dimensional data from $J$
component manifolds is available. If $N$ is large, then the above result
would suggest that we project each manifold into a lower-dimensional
subspace.  By collecting this data at a central location, we
would obtain $J$ vectors, each of dimension $O(K \log N)$, so that
we would have $O(JK \log N)$ total measurements.  This approach,
however, essentially ignores the {\em joint} manifold structure
present in the data.
If we instead view the data
as arising from a $K$-dimensional joint manifold residing in
$\reals^{JN}$ with bounded condition number as given by Theorem
\ref{thm:cond_jam}, we can then project the joint data into a subspace
which is {\em only logarithmic} in $J$ as well as the {\em largest}
condition number among the components, and still approximately
preserve the manifold structure. This is formalized in the following
theorem.
\begin{thm}
Let $\manifold^*$ be a compact, smooth, Riemannian joint
manifold in a $JN$-dimensional space with condition number
$1/\tau^*$. Let $\Phi$ denote an orthogonal linear mapping from
$\manifold^*$ into a random $M$-dimensional subspace of
$\reals^{JN}$. Let $M \geq O(K \log (JN/ \tau^*)/\epsilon^2)$. Then,
with high probability, the geodesic and Euclidean distances between
any pair of points on $\manifold^*$ are preserved up to distortion
$\epsilon$ under the linear transformation $\Phi$.
\label{thm:smashjam}
\end{thm}
Thus, we obtain a faithful approximation of our manifold-modeled
data that is only $O(K \log JN)$ dimensional. This represents a
significant improvement over performing separate dimensionality 
reduction on each component manifold. 

Importantly, the linear nature of the random projection step can
be utilized to perform dimensionality reduction in a distributed
manner, which is particularly useful in applications when data
transmission is expensive. As an example, consider a network of $J$
sensors observing an event that is governed by a $K$-dimensional
parameter. Each sensor records a signal $x_j \in \reals^N, 1\ \le j
\le J$; the concatenation of the signals $x =
[x_1^T~x_2^T~\ldots~x_j^T]^T$ lies on a $K$-dimensional joint
manifold $\manifold^* \subset \reals^{JN}$.
Since the required random projections are linear, we can
take local random projections of the observed signals at each
sensor, and still calculate the {\em global} measurements of
$\manifold^*$ in a distributed fashion. Let each sensor obtain its
measurements $y_j = \Phi_j x_j$, with the matrices $\Phi_j \in
\reals^{M \times N}, 1 \le j \le J$. Then, by defining the $M\times
JN$ matrix $\Phi = [\Phi_1 \ldots \Phi_J]$, our global projections
$y^* = \Phi^* x^*$ can be obtained by
\begin{eqnarray}
y^*&=& \Phi^* x^* \nonumber \\
&=& \Phi^* [x_1^T \quad x_2^T \quad \ldots \quad x_J^T]^T \nonumber \\
&=& [\Phi_1 \quad \Phi_2 \quad \ldots \quad \Phi_J] [x_1^T \quad x_2^T \quad \ldots \quad x_J^T]^T \nonumber \\
&=& \Phi_1 x_1 + \Phi_2 x_2 + \ldots + \Phi_J x_J . \nonumber
\end{eqnarray}
Thus, the final measurement vector can be obtained by simply {\em
adding independent random projections} of the signals acquired by
the individual sensors. This method enables a novel scheme for {\em compressive, multi-modal data fusion};
in addition, the number of random projections required by this scheme 
is only {\em logarithmic} in the number of sensors $J$. Thus, the joint 
manifold framework naturally lends itself to a network-scalable data 
aggregation technique for communication-constrained applications.

%% file: discussion.tex
\section{Discussion}
\label{sec:conc}

Joint manifolds naturally capture the structure present in a variety
of signal ensembles that arise from multiple observations of a
single event controlled by a small set of global parameters. We have
examined the properties of joint manifolds that are relevant to
real-world applications, and provided some basic examples that
illustrate how they improve performance and help reduce complexity.

We have also introduced a simple framework for dimensionality reduction
for joint manifolds that employs independent random projections from
each signal, which are then added together to obtain an accurate
low-dimensional representation of the data ensemble.
This distributed dimensionality reduction technique resembles the
acquisition framework proposed in compressive sensing
(CS)~\cite{Donoho04A, Candes04C}; in fact, prototypes of inexpensive
sensing hardware~\cite{spc,laska06} that can {\em directly acquire}
random projections of the sensed signals have already been built.
Our fusion scheme can be directly applied to the data acquired by
such sensors. Joint manifold fusion via random projections, like CS,
is {\em universal} in the sense that the measurement process is not
dependent on the specific structure of the manifold. Thus, our
sensing techniques need not be replaced for these extensions; only
our underlying models (hypotheses) are updated.

The richness of manifold models allows for the joint manifold
approach to be successfully applied in a larger class of problems
than principal component analysis and other linear model-based signal processing techniques.
In fact, joint manifolds can be immediately applied in signal
processing tasks where manifold models are common, such as
detection, classification, and parameter estimation. When these
tasks are performed in a sensor network or array, and random
projections of the captured signals can be obtained, joint manifold
techniques provide improved performance by leveraging the
information from all sensors simultaneously.